\documentclass{article}
\usepackage{ijcai20}

\usepackage{times}
\usepackage{soul}
\usepackage{url}
\usepackage[hidelinks]{hyperref}
\usepackage[utf8]{inputenc}
\usepackage[small]{caption}
\usepackage{graphicx}
\usepackage{amsmath}
\usepackage{amsthm}
\usepackage{booktabs}
\usepackage{algorithm}
\usepackage{algorithmic}
\usepackage{booktabs}
\usepackage{threeparttable}
\usepackage{multirow}
\usepackage{graphicx}
\newtheorem{theorem}{Theorem}
\newtheorem{lemma}{Lemma}

\title{Resolving Head-On Conflicts for Multi-Agent Path Finding with Conflict-Based Search}
\author{
	Lun Yang
	\affiliations
	Jilin University, China
	\emails
	yanglun17@mails.jlu.edu.cn
}

\begin{document}

\maketitle

\begin{abstract}
  Conflict-Based Search (CBS) is a popular framework for solving the Multi-Agent Path Finding problem. Some of the conflicts incur a foreseeable conflict in one or both of the children nodes when splitting on them. This paper introduces a new technique, namely the head-on technique that finds out such conflicts, so they can be processed more efficiently by resolving the conflict with the potential conflict all together in one split. The proposed technique applies to all CBS-based solvers. Experimental results show that the head-on technique improves the state-of-the-art MAPF solver CBSH.
\end{abstract}

\section{Introduction}

The Multi-Agent Path Finding (MAPF) problem is defined by a graph $G$ = ($V,E$) and a set of \textit{k} agents \textit{a}$_1$\textit{...a$_k$}, where each agent $a_i$ has a start location \textit{s$_i \in$ V} and a goal location \textit{g$_i \in$ V}. Time is discretized into timesteps, each agent can either \textit{move} to an adjacent location or \textit{wait} in its current location. Both move and wait actions incur a cost of one until the agent reaches its goal location and stays there. A \textit{path} of an agent $a_i$ is a sequence of move and wait actions that lead $a_i$ from its star location \textit{s$_i$} to its goal location \textit{g$_i$}. A tuple $\langle$$a_i$, $a_j$, $v$, $t$$\rangle$ is a \textit{vertex conflict} iff agent $a_i$ and $a_j$ occupy the same location $v$ $\in V$  at the same timestep $t$, and  A tuple $\langle$$a_i$, $a_j$, $u$, $v$, $t$$\rangle$ is an \textit{edge conflict} iff agents $a_i$ and $a_j$ travel the same edge ($u$, $v$) $\in E$ in opposite directions between timesteps $t$ - 1 and $t$. A \textit{solution} is a set of \textit{k} paths, one for each agent. Our task is to find a \textit{conflict-free} solution. Solving MAPF problem optimally has been proved to be NP-hard\cite{yu2013structure}\cite{yu2015intractability}. See \cite{felner2017search} for a survey.

\textit{Conflict-Based Search} (CBS) \cite{sharon2015conflict} is a two-level search-based optimal algorithm for MAPF which resolves conflicts by adding constraints on the involved agents. ICBS \cite{boyarski2015icbs} improved CBS with classifying and prioritizing conflicts. CBSH-CG (called \textsf{CG} here) \cite{felner2018adding} introduced an admissible heuristic for CBS high-level search for the first time by reasoning about a type of conflicts \cite{boyarski2015icbs} that increase cost to resolve. CBSH-WDG (called \textsf{WDG} here) \cite{li2019improved} purposed another admissible heuristic by reasoning about pairwise dependencies between agents. Rectangle reasoning technique \cite{li2019symmetry} identifies rectangle conflicts and resolves them efficiently in grid-based MAPF problem. Besides that a number of suboptimal CBS solvers have also been introduced \cite{Cohen2016Improved} \cite{barer2014suboptimal}. 

In this paper, we introduce a new way of reasoning about head-on conflicts for CBS-based MAPF solvers. Existing CBS-based solvers resolve head-on conflicts by making one of the involved agents wait for one timestep before the conflict, which will result in another predictable conflict. To this end, we resolve them in one split to improve the performance of CBS-based solvers.
\section{Background}

\subsection{Conflict-Based Search (CBS)}

CBS has two levels. The high level of CBS searches a binary \textit{conflict tree} (CT). Each CT node \textit{N} contains: \textbf{(1)} a set of constraints (\textit{N.constraints}) imposed on agents., where each constraint is either a \textit{vertex constraint} $\langle$$a_i$,  $v$, $t$$\rangle$ that prohibits agent $a_i$ from being at location $v$ at timestep $t$ or an \textit{edge constraint} $\langle$$a_i$, $u$, $v$, $t$$\rangle$ that prohibits agent $a_i$ from moving from location $u$ to $v$ between timesteps $t$-1 and $t$; \textbf{(2)} a solution (\textit{N.solution}) that satisfies \textit{N}.constraints; and \textbf{(3)} the cost of solution (\textit{N.cost}), which is the sum of the costs of all paths. The root CT node contains an empty set of constraints. The high level performs a best-first search according to the costs of CT nodes. 

When CBS chooses a CT node \textit{N} to expand, it checks \textit{N.solution} for conflicts. If there are none, CBS terminates, declares \textit{N} the goal node and returns \textit{N.solution}. Otherwise, CBS chooses a conflict randomly to resolve, by \textit{splitting N} into two child CT nodes. In each child CT node, an additional constraint is added to prohibit one of the involved agents from occupying the location $v$ or traversing the edge ($u$,$v$)at $t$. The path of the agent must be replanned by a low-level search since it no longer satisfies the constraints of the child CT node. By considering both ways to resolve each conflict, CBS guarantees optimality.

\subsection{Improved CBS (ICBS)}

CBS arbitrarily chooses conflicts to split on, while the order of resolving conflicts can significantly affect the size of the CT and thus the runtime. \textit{Improved CBS}(ICBS)\cite{boyarski2015icbs} addresses this issue by classifying and prioritizing conflicts. ICBS chooses conflicts in order of: \textbf{(1)} \textit{cardinal} conflicts, a conflict is cardinal iff when CBS splits a CT node \textit{N} on it, the cost of each child CT node of \textit{N} is larger than \textit{N.cost}.  \textbf{(2)} \textit{semi-cardinal} conflicts, a conflict is semi-cardinal iff the cost of one child CT node of \textit{N} is larger than \textit{N.cost}
and the cost of the other is equal to \textit{N.cost}. \textbf{(3)} \textit{non-cardinal } conflicts, a conflict is non-cardinal iff the cost of each child CT node of \textit{N} is equal to \textit{N.cost}.

ICBS uses MDDs to classify conflicts. The \textit{multi-valued decision diagram}(MDD)  \cite{sharon2013increasing} is a directed acyclic graph that  compactly stores all possible paths of given cost \textit{c} for a given agent $a_i$ while only considering the constraints imposed on $a_i$. Nodes at depth $t$ of $a_i$'s MDD are all possible locations along the optimal paths that cost \textit{c} at timestep $t$. If there is only one node $\langle v, t\rangle$ at depth $t$ it is called a \textit{singleton}, which means all optimal paths of that agent must be at $v$ at timestep $t$. Cardinal conflicts happen between two singletons, semi-cardinal conflicts involve only one. Figure 1 gives an example, the vertex conflict $\langle a_1, a_2, B2, 3 \rangle$  between (b) and (c) is cardinal since $\langle B2, 3\rangle$ is a singleton in both MDDs. The edge conflict $\langle a_1, a_2, B2, B3$, 3$\rangle$ between (a) and (b) is cardinal since $\langle B2, 2\rangle$, $\langle B3, 3\rangle$ are singletons in $a_1$'s MDD and $\langle B3, 2\rangle$, $\langle B2, 3\rangle$ are singletons in $a_2$'s MDD. 

\begin{figure}
	\includegraphics[scale=0.26]{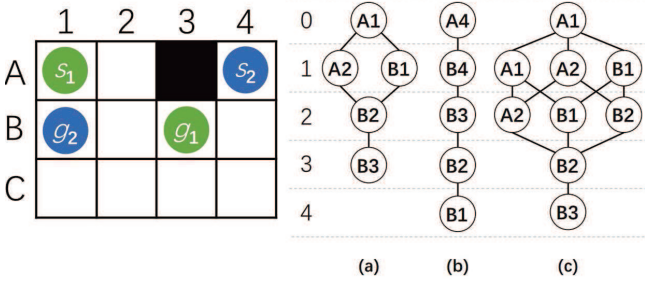}
	\caption{An example of head-on conflicts and both agents' MDD. (a) and (b) are MDDs of $a_i$ and $a_j$ before the CBS split, (c) is the MDD of $a_i$ after the CBS split.}
\end{figure}

\section{IDENTIFY HEAD-ON CONFLICTS}
In this section, we further classify cardinal conflicts into three types. Like classifying the cardinal and semi-cardinal conflicts needs to check the MDD width of both agents separately, we reason about the involved agents one at a time.
\subsection{Definition of Head-on} 
When classifying a conflict $\langle$$a_i$, $a_j$, $v$, $t$$\rangle$$/$ $\langle$$a_i$, $a_j$, $u$, $v$, $t$$\rangle$, we say the conflict is \textit{head-on} for $a_i$ iff it satisfies the following conditions:
\begin{enumerate}
	\item The conflict is cardinal.
	\item The conflict occurs before both agents reach their goal locations, and if the conflict is a vertex conflict, $a_i$'s MDD has a singleton $\langle v_{pre}, t-1\rangle$ and $a_j$'s MDD has a singleton $\langle v_{pre}, t+1\rangle$ where $v_{pre}$ is $a_i$'s location at timestep $t-1$ in its optimal path before the split.
	\item After adding the constraints \{$\langle$$a_i$, $v$, $t$$\rangle$, $\langle$$a_i$, $v_{pre}$, $v$, $t + 1$$\rangle$\}$/$ \{$\langle$$a_i$, $u$, $v$, $t$$\rangle$, $\langle$$a_i$, $u$, $t$$\rangle$\} the length of $a_i$'s path is larger than $l + 1$ (where $v_{pre}$ is $a_i$'s location at timestep $t - 1$ in $N$.solution, $l$ is the length of $a_i$'s path before the split.).
\end{enumerate}

Condition(1) and condition(2) intend to make sure that the optimal paths of the two agents have a conflict on a little ``footlog'', which means the conflict can't be resolved if we only make $a_i$ wait for on timestep before the conflict. We use Condition(2) to check whether the two agents are heading in opposite directions and whether the conflict is  ``surrounded'' by singletons. There are extra conditions for vertex conflicts because cardinal edge conflicts already suit such requirements. Condition(3) is imposed for excluding the influence of existing constraints. With all these three conditions, now we can say there will be another cardinal conflict after resolving the first one for sure. 

\begin{theorem} \label{theorem1}
	Suppose that a CT node N chooses to resolve a vertex/ edge conflict $\langle$$a_i$, $a_j$, $v$, $t$$\rangle$$/$ $\langle$$a_i$, $a_j$, $u$, $v$, $t$$\rangle$ that is head-on for $a_i$, and the child CT node of N, N$_c$ (with an additional constraint  $\langle$$a_i$, $v$, $t$$\rangle$$/$ $\langle$$a_i$, $u$, $v$, $t$$\rangle$) has a solution, then N$_i$ has a cardinal edge/ vertex conflict $\langle$$a_i$, $a_j$, $v_{pre}$, $v$, $t + 1$$\rangle$$/$ $\langle$$a_i$, $a_j$, $u$, $t$$\rangle$.
\end{theorem}
\begin{proof}
	We prove the theorem by adding the two constraints separately, so there will be two replannings in this process. 
	After adding the first constraints $\langle$$a_i$, $v$, $t$$\rangle$$/$ $\langle$$a_i$, $u$, $v$, $t$$\rangle$, \{$\langle v_{pre},t\rangle$, $\langle v,t + 1\rangle$\}$/$ $\langle u,t\rangle$ are/is singleton(s) in the new MDD, otherwise there will be two situations: the nodes are not in the new MDD or there are other nodes at the same depth in the new MDD. The conflict occurs before both agents reach their goal locations, in such cases \cite{li2019improved} has proven that $N_i$.cost $\in$ \{$N$.cost, $N$.cost + 1\}. So the length of the replanned path for $a_i$ is $l$ or $l + 1$. Since the conflict is cardinal the length of the first replanned path must be $l + 1$ and so be the depth of the new MDD. In the first situation, the second constraint can't affect the second replanning, so the length of $a_i$'s path is still $l+1$ contradict the condition(3). In the second situation, there are other edges/vertexes at the same depth which means after adding the second constraint there are other options to reach the goal with the same timesteps $l + 1$. In other words, the second constraint can't increase the length of $a_i$'s path, therefore both situations contradict the assumption.
	
	If the conflict is a vertex conflict, $\langle v_{pre} , t \rangle$ and $\langle v, t + 1\rangle$ are singletons in $a_i$ 's new MDD, $\langle v , t \rangle$ and $\langle v_{pre}, t + 1 \rangle$ are still singletons in $a_j$'s MDD in $N_i$ according to condition(2). So $N_i$ has a cardinal edge conflict $\langle$$a_i$, $a_j$, $v_{pre}$, $v$, $t + 1$$\rangle$.
	If the conflict is an edge conflict, $\langle u, t\rangle$ is a singleton in both agent's MDD in $N_i$, Therefore $N_i$ has a cardinal vertex conflict $\langle$$a_i$, $a_j$, $u$, $t$$\rangle$.
\end{proof}

\begin{figure}
	\includegraphics[scale=0.3]{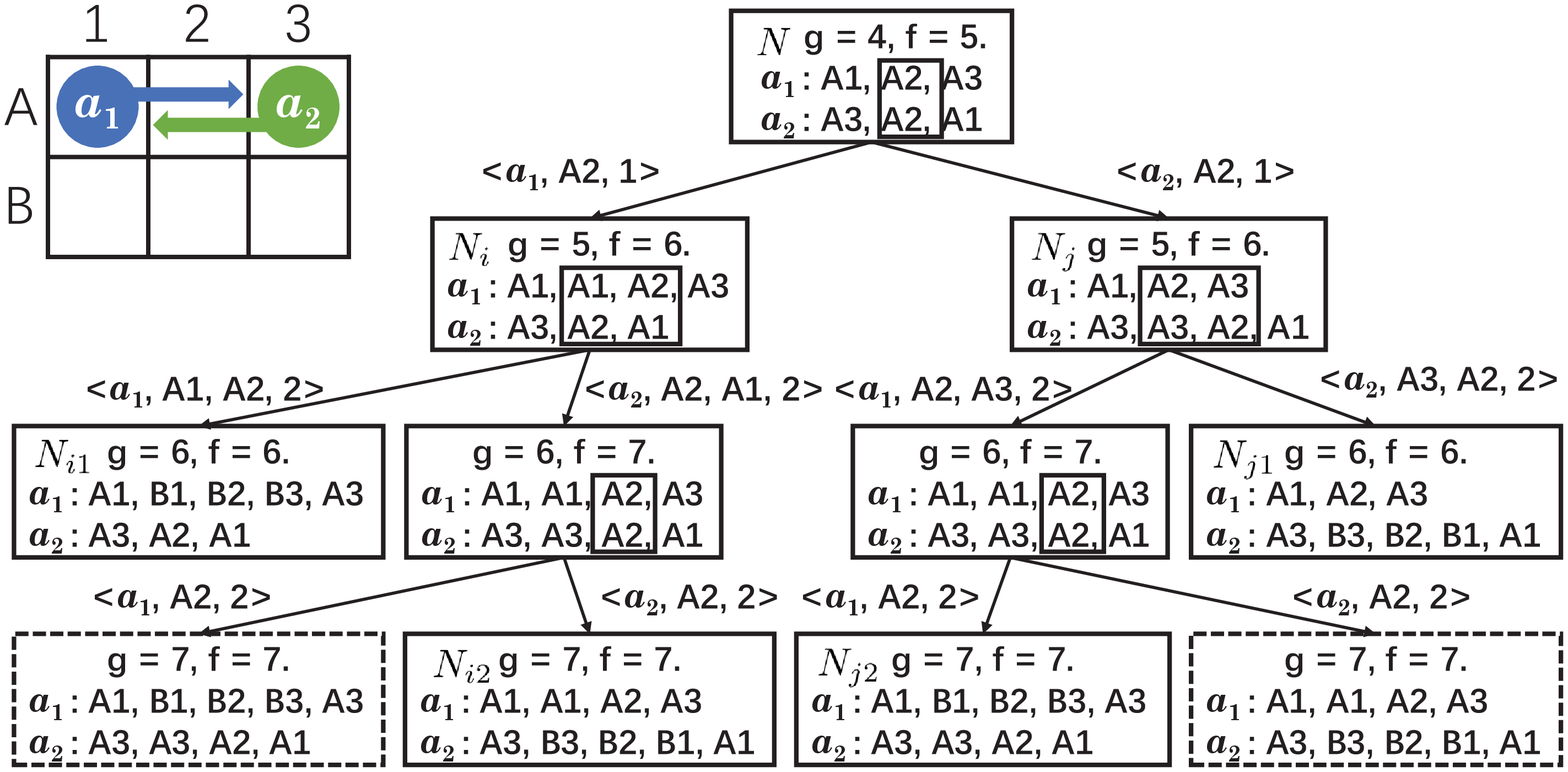}
	\caption{A MAPF instance and its CT.}
\end{figure}

For better understanding we give an example shown in Figure 1. $\langle a_1, a_2, B2, B3, 3\rangle$ is a cardinal edge conflict. $\langle B2, 2\rangle$ is a singleton in $a_1$'s MDD,  After adding the constraint $\langle a_1 ,B2, B3, 3\rangle$, $\langle B2, 2\rangle$ is delayed to $\langle B2, 3\rangle$ in the MDD shown in Figure 1(c). $\langle B2, 3\rangle$ is a singleton in $a_2$'s MDD, which resulting in a cardinal vertex conflict $\langle a_1, a_2, B2, 3\rangle$. In Figure 1. $a_1$ and $a_2$ have a vertex conflict $\langle a_1, a_2, A2, 1\rangle$ that is head-on for both  agents. The location $A1$ is $v_{pre}$ for  $a_i$. After adding the constraint $\langle a_1 ,A2, 1\rangle$ there are singletons $\langle A1, 1\rangle$ and $\langle A2, 2\rangle$ in $a_i$'s new MDD, singletons $\langle A2, 1\rangle$ and $\langle A1, 2\rangle$ remain in $a_2$'s MDD. So there must be a cardinal edge conflict $\langle a_1 , a_2, A1, A2, 2\rangle$.

After the definition of head-on, we formally define head-on conflicts.

\subsection{Definition of Head-on Conflicts}

\paragraph{Head-on conflict.}A cardinal conflict is head-on for a CT node \textit{N} iff it is head-on for both involved agents. When CBS splits on a head-on conflict of \textit{N}, both children nodes will have a known cardinal conflict. For example, in Figure 2 the conflict $\langle a_1, a_2, A2, 1\rangle$ is a head-on conflict for the root node. Both children node of the root node have a cardinal conflict $\langle a_1, a_2, A1, A2, 2\rangle$ or $\langle a_1, a_2, A2, A3, 2\rangle$ that is predictable before the split action.

\paragraph{Semi-head-on conflict.}A cardinal conflict is semi-head-on for a CT node \textit{N} iff it is head-on for one of the involved agents. When CBS splits on a semi-head-on conflict of \textit{N}, one of the children nodes will have a known cardinal conflict. For example, in Figure 3(a) the conflict $\langle a_1, a_2, B2, 1\rangle$ is a semi-head-on conflict for the root node. The child node with the new constraint $\langle a_1, B2, 1\rangle$ have a cardinal conflict $\langle a_1, a_2, B1, B2, 2\rangle$ that is predictable before the split action.

\paragraph{Non-head-on conflict.}A cardinal conflict is non-head-on for a CT node \textit{N} iff it is not head-on for both involved agents. For example, in Figure 3(b) the conflict $\langle a_1, a_2, B2, 1\rangle$ is a non-head-on conflict for the root node. There are no predictable conflicts can be found with head-on technique in both of its children nodes.

When using the head-on technique, head-on conflicts are given the highest priority to resolve, and then the semi-head-on, at last the ordinary cardinal conflict and the rest is identical to ICBS.

\subsection{A Technique to Accelerate Identification}
The condition (3) in the definition of head-on uses the low-level planner to identify head-on conflicts. Actually, we don't have to do that for all agents that satisfy the condition(1) and (2), but only those who already have constraints on them.  Now we explain why:

Paths can be seen as a sequence of actions that lead the agent from its start location to its goal location. In a 4-neighbor grid map there are five kinds of actions: \textit{left, right, up, down} and \textit{wait}, we use  $n_{left}$, $n_{right}$, $n_{up}$, $n_{down}$ and $n_{wait}$ to denote the number of an action in the original path, $n_{left}^\prime$, $n_{right}^\prime$, $n_{up}^\prime$, $n_{down}^\prime$ and $n_{wait}^\prime$ to denote the number of an action in the replanned path respectively. The horizontal distance between the start location and goal location is equal to $|n_{left} - n_{right}|$ and the vertical distance is equal to $|n_{up} - n_{down}|$. Obviously the optimal path without constraints has no wait actions. As proved in Theorem 1 the length of $a_i$'s replanned path with the first constraint is $l + 1$. So the replanned path has only one wait action. Apparently one constraint can't create more than one wait actions, we assume the only opposite that there exists a replanned path that has no wait actions with a length of $l+1$, then we have:

\begin{equation}
\begin{aligned}
|n_{left} - n_{right}|= |n_{left}^\prime - n_{right}^\prime|
\end{aligned}\end{equation}
\begin{equation}
\begin{aligned}
|n_{up} - n_{down}|= |n_{up}^\prime - n_{down}^\prime|
\end{aligned}\end{equation}\begin{equation}
\begin{aligned}
n_{left} + n_{right} + n_{up} + n_{down} + + n_{wait}= l
\end{aligned}\end{equation}\begin{equation}
\begin{aligned}
n_{left}^\prime + n_{right}^\prime + n_{up}^\prime + n_{down}^\prime + n_{wait}^\prime = l + 1
\end{aligned}\end{equation}

$n_{wait}$ and $n_{wait}^\prime$ are both 0 here. We use  $\Delta_{left}$, $\Delta_{right}$, $\Delta_{up}$, $\Delta_{down}$ to denote the differences of number of actions after replan. So we have:
\begin{equation}
\begin{aligned}
\Delta_{left}+ \Delta_{right}+ \Delta_{up}+ \Delta_{down} =  1
\end{aligned}\end{equation}
\begin{figure}
	\centering{\includegraphics[scale=0.2]{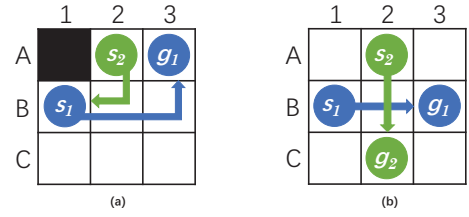}}
	\caption{Examples of semi-head-on conflicts and non-head-on conflicts.}
\end{figure}

From Equation (1) and (2) we know that $\Delta_{left} - \Delta_{right}$ is 0 or -2($n_{left} - n_{right}$) which is even,  $\Delta_{up} - \Delta_{down}$ is 0 or -2$(n_{up} - n_{down})$ which is even as well, then we have:		
\begin{equation}
\begin{aligned}
\Delta_{left} - \Delta_{right} + 2\Delta_{right} + \Delta_{up} - \Delta_{down} + 2\Delta_{down} = 1
\end{aligned}\end{equation}

All subitems on the left side of Equation(6) are even, the right side is odd, therefore the equation doesn't hold, neither the assumption. So the replanned path has only one wait action.

Wait action can't change the location, so the left replanning problem is identical to the original planning, that is finding a sequence of $l$ actions from left, right, up and down four actions that connects the start and the goal location, so they have the same set of resulting paths, the only difference is that there will be a wait action somewhere in the sequences before timestep $t$ for satisfying the constraint. By assumption both $\langle v_{pre} , t - 1 \rangle$$/$ $\langle u, t - 1 \rangle$ and $\langle v, t\rangle$ are singletons, which means all paths must traverse $v_{pre}$ at timestep $t - 1$ and $v$ at timestep $t$. After the insertion, the singletons are delayed to $\langle v_{pre} , t \rangle$$/$ $\langle u, t \rangle$ and $\langle v, t + 1\rangle$ with the first constraint. That leads to the same conclusion as condition(3), so there are no constraints on the agent can be seen as a much cheaper substitute for condition (3).

The example shown in Figure 1 still works here. There are no constraints on $a_1$ before, after adding the constraint $\langle a_1, B2, B2, 3\rangle$ the singletons $\langle B2, 2\rangle$ and $\langle B3, 3\rangle$ in Figure 1(a) are delayed to singletons $\langle B2, 3\rangle$ and $\langle B3, 4\rangle$ in Figure 1(c).

\section{Resolving Head-on Conflicts}

  As mentioned above, the head-on technique works by resolving original conflicts and those predictable conflicts in one split, to reduce the depth of the goal nodes and the size of CT. In other words, our goal is to make sure all children nodes are free from known conflicts without losing any conflict-free solutions. To that purpose, we handle vertex and edge head-on conflicts differently.  

\subsection{Semi-head-on Edge Conflicts} 
When a CT node $N$ split on a semi-head-on edge conflict $\langle a_i, a_j, u, v, t\rangle$ which is head-on for $a_i$, the head-on technique adds constraints $\langle a_i, u, v, t\rangle$, $\langle a_i, u, t\rangle$ to one child CT node $N_i$ and $\langle a_j, v, u, t\rangle$ to the other. In another word we bring forward the constraint $\langle a_i, u, t\rangle$ that is going to be added somewhere in the subtree of the node $N_i$. Both $N$'s children nodes are free from the conflicts  $\langle a_i, a_j, u, v, t\rangle$ and $\langle a_i, a_j, v, t\rangle$.
\begin{figure*}
	\centering
	\includegraphics[scale=0.62]{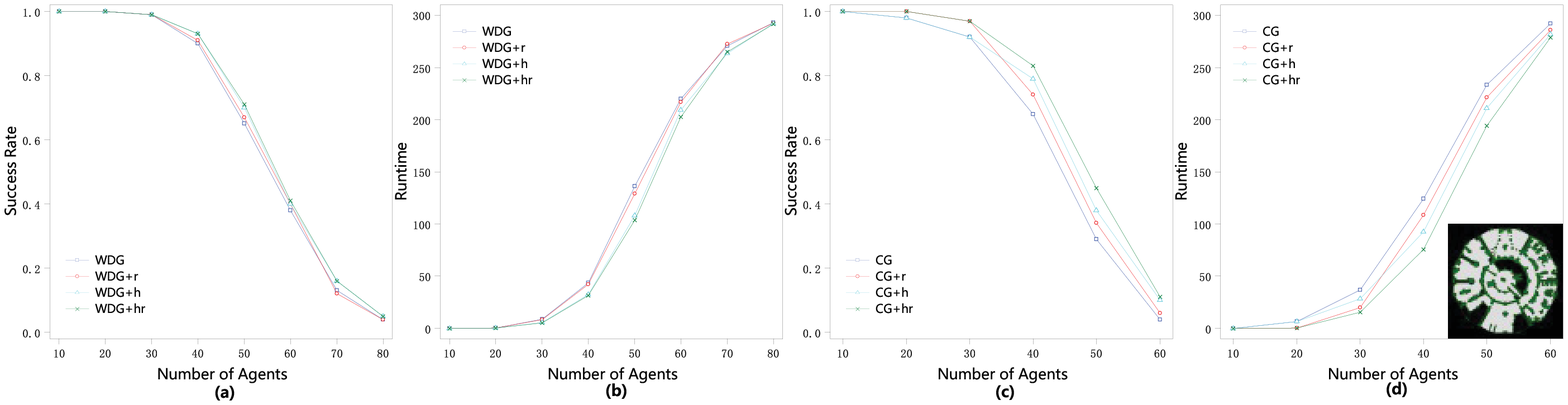}
	\caption{Results of \text{WDG} and \textsf{CG} with different enhancements on \textit{lak503d}.}
\end{figure*}

\begin{figure*}
	\centering
	\includegraphics[scale=0.62]{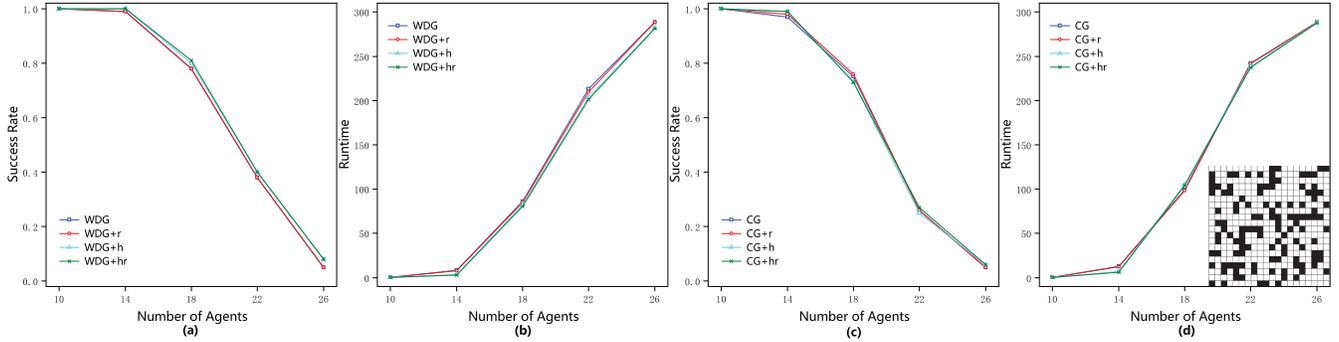}
	\caption{Results of \text{WDG} and \textsf{CG} with different enhancements on 20$\times$20 grids with 30$\%$ obstacles. The blue lines in (a) and (d) are hidden by the red lines.}
\end{figure*}
\begin{lemma}
	If two agents have an edge conflict $\langle a_i, a_j, u, v, t\rangle$, then there are no pairs of conflict-free paths for the two agents that violate both sets of constraints $C_1$ = \{$\langle a_i, u, v, t\rangle$, $\langle a_i, u, t\rangle$\} and $C_2$ = \{$\langle a_j, v, u, t\rangle$\}.
\end{lemma}

\begin{proof}
	If there is such a pair of conflict-free paths, then $a_j$ must be at $v$ at timestep $t-1$ and at $u$ at timestep $t$ since it violates $C_2$. It also needs to make sure $a_i$ be at $u$ at timestep $t-1$ and at $a_j$ at timestep $v$, or be at $u$ at timestep $t$ for the same reason. Both situations described above end up with a conflict, which contradicts the assumption. 
\end{proof}

\subsection{Head-on Edge Conflicts} 
When a CT node $N$ split on a head-on edge conflict $\langle a_i, a_j, u, v, t\rangle$, the head-on technique adds constraints $\langle a_i, u, v, t\rangle$, $\langle a_i, u, v, t + 1\rangle$ to one child CT node and $\langle a_j, v, u, t\rangle$, $\langle a_j, v, u, t + 1\rangle$ to the other. We can't handle head-on edge conflicts the same way as semi-head-on conflicts because there is a possibility that both agents wait right before the conflict happens. To this end, we ``delay'' the second constraints to include that possibility to the search spaces of both children nodes' subtree. Both $N$'s children nodes are free from the conflicts $\langle a_i, a_j, u, v, t\rangle$, $\langle a_i, a_j, u, t\rangle$ and $\langle a_i, a_j, u, t\rangle$.

\begin{lemma}
	If two agents have an edge conflict $\langle a_i, a_j, u, v, t\rangle$, then there are no pairs of conflict-free paths for the two agents that violate both sets of constraints $C_1$ = \{$\langle a_i, u, v, t\rangle$, $\langle a_i, u, v, t + 1\rangle$\} and $C_2$ = \{$\langle a_j, v, u, t\rangle$, $\langle a_j, v, u, t + 1\rangle$\}.
\end{lemma}

\begin{proof}
We are going to analyze all combinations of constraints between $C_1$ and $C_2$. Violating constraints $\langle a_i, u, v, t\rangle$ and $\langle a_j, v, u, t\rangle$ at the same time results in a conflict $\langle a_i, a_j, u, v, t\rangle$. Violating constraints $\langle a_i, u, v, t\rangle$ and $\langle a_j, v, u, t + 1\rangle$ at the same time results in a conflict $\langle a_i, a_j, v, t\rangle$. Violating constraints $\langle a_i, u, v, t + 1\rangle$ and $\langle a_j, v, u, t\rangle$ at the same time results in a conflict $\langle a_i, a_j, u, t\rangle$. Violating constraints $\langle a_i, u, v, t+1\rangle$ and $\langle a_j, v, u, t+1\rangle$ at the same time results in a conflict $\langle a_i, a_j, u, v, t+1\rangle$. Therefore all combinations of constraints between $C_1$ and $C_2$ can't be violated without conflicts and the lemma is proven. 
\end{proof}

\subsection{Head-on and Semi-head-on Vertex Conflicts} 
A head-on or semi-head-on vertex conflict $\langle a_i, a_j, v, t\rangle$ is handled as follow: if the conflict is head-on for $a_i$ we add constraints $\langle a_i, v, t\rangle$, $\langle a_i, v_{pre}, v, t+1\rangle$ to the child CT node $N_i1$, and an additional child CT node $N_i2$ with new constraints $\langle a_i, v, t\rangle$, $\langle a_j, v, v_{pre}, t+1\rangle$ $\langle a_j, v, t+1\rangle$. If the conflict is not head-on for $a_i$, then we add a child CT node $N_i$ with new constraint $\langle a_i, v, t\rangle$, the same as original CBS does. So splitting on a head-on vertex conflict generates four children nodes, generates three if it is a semi-head-on vertex conflict. All of $N$'s children nodes are free from the conflicts between the two agents around $v$.

Now we explain the reason why we add additional nodes. Figure1 gives a good example on this issue. There are six CT nodes that come out of resolving the head-on vertex conflict $\langle a_1, a_2, A2, 1\rangle$ and its predictable following conflicts. Only the cost of $N_i$ and $N_j$ is 6, while all the other four nodes' cost is 7. Just like resolving head-on edge conflict directly will leave a possibility behind, there are two possibilities corresponding the two constraint combinations. One is $a_i$ is at $v_{pre}$ at timestep $t$ and at $v$ at timestep $t+1$, $a_j$ is at $v$ at timestep $t$, the other is $a_j$ is at $v_{post}$ at timestep $t$ and at $v$ at timestep $t+1$, $a_i$ is at $v$ at timestep $t$($v_{post}$ is $a_j$'s location at timestep $t-1$ before the split, corresponding to $v_{pre}$). These two possibilities can't be covered by delaying the second constraint. In the CT of original CBS, the two possibilities are considered in $N_{i2}$ and $N_{j2}$ as shown in Figure 1. So due to completeness, we add an additional CT node whenever the vertex conflict we resolve is head-on for an involved agent. The additional nodes mean to take those possibilities into consideration.

To sum up, if the vertex conflict is semi-head-on and only head-on for one agent let's say $a_i$, then we generate three children node $N_{i1}$, $N_{i2}$, and $N_{j}$, if it is head-on then we generate four children node $N_{i1}$, $N_{i2}$, $N_{j1}$ and $N_{j2}$.

\begin{lemma}
	If twe agents have a vertex conflict $\langle a_i, a_j, v, t\rangle$ that is head-on for $a_i$, then there are no pairs of conflict-free paths for the two agents that violate all three sets of constraints $C_1$ = \{$\langle a_i, v, t\rangle$, $\langle a_i, v_{pre}, v, t + 1\rangle$\}, $C_2$ = \{$\langle a_j, v, t\rangle$\} and $C_3$ = \{$\langle a_i, v, t\rangle$, $\langle a_j, v, v_{pre}, t + 1\rangle$, $\langle a_j, v, t + 1\rangle$\}.
\end{lemma}
\begin{proof}
	We analyze constraint combinations between constraint sets $C_1$ and $C_2$. Constraint combination \{$\langle a_i, v, t\rangle$, $\langle a_j, v, t\rangle$\} can't be violated at the same time without conflict. That leaves only one combination \{$\langle a_i, v_{pre}, v, t+1\rangle$, $\langle a_j, v, t\rangle$\}, which means $a_i$ must be at $v_{pre}$ at timestep $t$ and at $v$ at timestep $t+1$, $a_j$ must be at $v$ at timestep $t$. Now we only need to prove there must be a conflict if it violates $C_3$. It can't violate $\langle a_i, v, t\rangle$ obviously. There will be a conflict $\langle a_i, a_j, v_{pre}, v, t+1\rangle$ if it violates $\langle a_j, v, v_{pre}, t + 1\rangle$, and a conflict $\langle a_i, a_j, v, t + 1\rangle$ if it violates $\langle a_j, v, t + 1\rangle$. Therefore there are no pairs of conflict-free paths for the two agents that violate all three sets of constraints.
\end{proof}
\begin{lemma}
	If twe agents have a head-on vertex conflict $\langle a_i, a_j, v, t\rangle$, then there are no pairs of conflict-free paths for the two agents that violate all four sets of constraints $C_1$ = \{$\langle a_i, v, t\rangle$, $\langle a_i, v_{pre}, v, t + 1\rangle$\}, $C_2$ = \{$\langle a_j, v, t\rangle$, $\langle a_j, v_{post}, v, t + 1\rangle$\}, $C_3$ = \{$\langle a_i, v, t\rangle$, $\langle a_j, v, v_{post}, t + 1\rangle$, $\langle a_j, v, t + 1\rangle$\} and $C_4$ = \{$\langle a_j, v, t\rangle$, $\langle a_i, v, v_{post}, t + 1\rangle$, $\langle a_i, v, t + 1\rangle$\}. 
\end{lemma}
\begin{proof}
	 
	We analyze constraint combinations between constraint sets C1 and C2. Apart from the proof on Lemma 3, constraint combination \{$\langle a_i, v_{pre}, v, t+1\rangle$, $\langle a_j, v_{post}, v, t+1\rangle$\} can't be violated together so there is only one new constraint combination \{$\langle a_i, v, t\rangle$, $\langle a_j, v_{post}, v, t + 1\rangle$\} between two sets that we need to prove. It means that $a_i$ must be at $v$ at timestep $t$, $a_j$ must be at $v_{post}$ at timestep $t$ and at $v$ at timestep $t+1$. Now we prove that it can't violate $C_4$. It can't violate $\langle a_j, v, t\rangle$ obviously. There will be a conflict $\langle a_i, a_j, v, v_{post}, t+1\rangle$ if it violates $\langle a_i, v, v_{post}, t + 1\rangle$, and a conflict $\langle a_i, a_j, v, t + 1\rangle$ if it violates $\langle a_i, v, t + 1\rangle$. Therefore there are no pairs of conflict-free paths for the two agents that violate all four sets of constraints.
\end{proof}
 \begin{figure*}
	\centering
	\includegraphics[width=8.5cm]{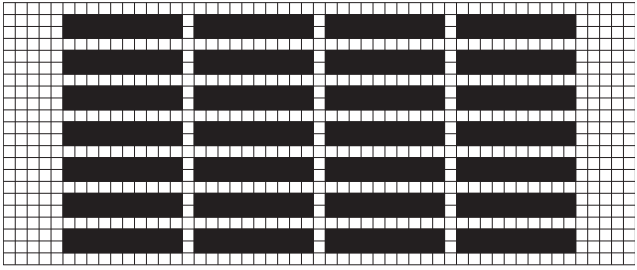} 
	\hspace{3ex}									
	\includegraphics[width=8.5cm]{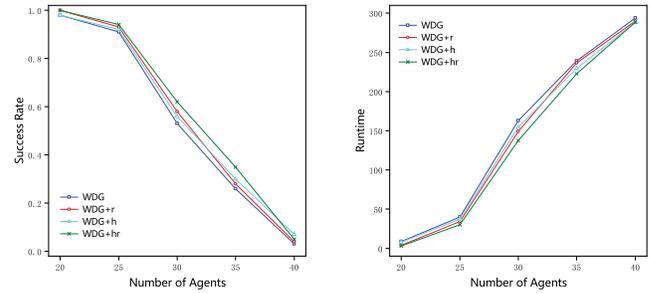} 
	
	\caption{Left: Automate warehouse map. Right: Success rate and runtime of \textsf{WDG} with different enhancements on warehouse map}
\end{figure*}
\begin{theorem}
	CBS with head-on technique is complete and optimal.
\end{theorem}
\begin{proof}
	The low-level always plan an optimal path, the high-level always chooses the node with minimum $f$-value to expand. Lemma 1, Lemma 2, Lemma 3 and Lemma 4 have proven that the head-on technique keeps all the conflict-free solutions in split actions. So the first chosen CT node without conflicts has the minimum cost, which means CBS with head-on technique is optimal. The cost of the expanded node is non-decreasing and there is a finite number of CT nodes with the same cost. So if there is a solution without conflicts that cost $c$, before all node with a solution that cost smaller than $c+1$ been expanded, the solution must be found. i.e. CBS with head-on technique is complete.
\end{proof}

Like both children nodes' cost will be increased by at least one, head-on conflicts guarantee the cost increased by at least two. Because we resolve the original and the foreseeable conflicts in one split, the children nodes' costs are increased by at least one twice. Such property can be used in heuristics calculation of CT nodes.  

\section{Experiment}

In this section, we compare the performance of CBSH with and without the head-on technique on the benchmark map, randomly generated small grids and warehouse maps. And also we compare the head-on technique with a similar CBS enhancement --- rectangle reasoning technique. \textsf{WDG}+r (\textsf{WDG} with rectangle reasoning technique) is the state-of-the-art CBS-based solver in the previous research. All the test cases(agents' start and goal locations) in this section are generated with random numbers, for each map and number of agents we generate 100 instances. The head-on technique is implemented with the identification technique mentioned in section 3. The experiments are conducted on a 1.90 GHz Intel Core i7-8650U laptop with 16 GB RAM with 5 minutes runtime limit.

\subsection{Benchmark Maps}

Figure 4 shows the success rates and runtimes of \textsf{WDG} and \textsf{CG} with different enhancements on benchmark grid map \textit{lak503d}\cite{Sturtevant2012Benchmarks}, which is a 194$\times$194 grid map. In Figure 4, ``+h'' means with the head-on technique, ``+r'' means with the rectangle reasoning technique. Both enhancements increase the success rate of \textsf{WDG}. The head-on technique's performance is better than the rectangle reasoning technique on \textsf{WDG}. The rectangle reasoning technique works better with \textsf{CG} then the number of agents is small and then the head-on technique take over after agents number is larger than 40. Both algorithms perform even better with two enhancements together. 

Table 1 presents the runtime and CT node expanded of the cases solved by both \textsf{WDG} + r and \textsf{WDG} + hr. With the head-on technique the nodes expanded is reduced by more than a half, and so is the runtime. Just as we intended to improve the performance by reducing nodes expanded and the depth of CT.
\subsection{Small Maps}

Figure 5 shows the success rates and runtimes of \textsf{WDG} and \textsf{CG} with different enhancements on small maps. The map is 20$\times$20 with 30$\%$ of randomly generated obstacles. The rectangle technique compromised on \textsf{WDG} while the head-on technique slightly improves the performance. The two enhancements' efficiencies are roughly the same on \textsf{CG} .

\begin{table}[]
	\resizebox{8.5cm}{1.1cm}{
	\begin{tabular}{cccccc}
		\hline
		\multicolumn{1}{l}{\multirow{2}{*}{Agents}} & \multicolumn{1}{l}{\multirow{2}{*}{Cases}} & \multicolumn{2}{c}{WDG+r}                               & \multicolumn{2}{c}{WDG+hr}                              \\ \cline{3-6} 
		\multicolumn{1}{l}{}                        & \multicolumn{1}{l}{}                       & \multicolumn{1}{l}{Runtime} & \multicolumn{1}{l}{Nodes} & \multicolumn{1}{l}{Runtime} & \multicolumn{1}{l}{Nodes} \\ \hline
		30                                          & 99                                         & 5.26                        & 384.71                    & 2.226                       & 127.89                    \\ \hline
		40                                          & 91                                         & 19.38                       & 1417.99                   & 8.88                        & 422.21                    \\ \hline
		50                                          & 65                                         & 41.04                       & 2672.66                   & 19.81                       & 929.72                    \\ \hline
	\end{tabular}
}
\caption{Results on \textit{lak503d}. The first column ``Agents''
	shows the number of agents. ``Cases'' indicates the number of instances solved by both \textsf{WDG}+r and \textsf{WDG}+hr in 5 minutes. The rest columns show the average runtime and CT nodes expanded in those cases.}
\end{table}

\subsection{Warehouse Maps}

 Figure 6 presents the warehouse map we used in experiments. Warehouse maps means to simulate how the automated warehousing works. All test cases move an agent from the left five columns to the right or the other way around. Like the result on the benchmark map, both enhancements improve the performance of \textsf{WDG} and the rectangle reasoning technique is more effective while the number of agents is small. As the number of agents gets larger the head-on technique works batter and beat \textsf{WDG}+hr on success rate when there are 40 agents.

\subsection{Inefficiency on Empty Maps}
We also ran a test on a 20$\times$20 empty map and the head-on technique offers nearly no improvements on the results. Such inefficiency is predictable because agents are more flexible which means it rather hard to form a head-on conflict. Therefore the head-on technique works poorly on maps with little obstacles.
 \section{Conclution and Future Work}

In this paper we identify head-on conflicts from cardinal conflicts, then we demonstrate that the way CBS split on head-on conflicts can be improved. So we introduce more efficient split methods for different types of head-on conflicts. Experimental results report larger improvements than the rectangle reasoning technique on CBSH-WDG and CBSH-CG. Besides that algorithms with the two enhancements together work even better.

We suggest the following future research directions:
(1) Better way of resolving head-on conflicts. For example, use positive constraints \cite{li2019disjoint} to reduce the search space of additional children CT nodes' subtrees; (2) computing better heuristics with head-on conflicts; and (3) apply the head-on technique to suboptimal MAPF solvers.
In experiments, splitting head-on vertex conflicts without additional nodes report great improvement in both success rate and runtime.

\bibliographystyle{named}
\bibliography{ijcai20}

\end{document}